\documentclass{article}
\usepackage{ijcai23}
\usepackage{pdfpages}

\usepackage{times}
\usepackage{soul}
\usepackage{url}
\usepackage[hidelinks]{hyperref}
\usepackage{graphicx}
\usepackage{amsmath}
\usepackage{amsthm}
\usepackage{booktabs}
\usepackage{algorithm}
\usepackage{algorithmic}
\usepackage[switch]{lineno}

\usepackage{algorithm}
\usepackage{algorithmic}

\usepackage[utf8]{inputenc} 
\usepackage[T1]{fontenc}    
\usepackage{hyperref}       
\usepackage{url}            
\usepackage{booktabs}       
\usepackage{amsfonts}       
\usepackage{nicefrac}       
\usepackage{microtype}      
\usepackage{xcolor}         

\usepackage{graphicx}

\usepackage{amsmath}
\usepackage{amsthm}
\usepackage{booktabs}
\usepackage{multirow}
\usepackage{esint}
\newtheorem{definition}{Definition}
\newtheorem{lemma}{Lemma}
\newtheorem{assumption}{Assumption}

\title{Negative Flux Aggregation to Estimate Feature Attributions}

\author{
Xin Li
\and
Deng Pan\and
Chengyin Li\and
Yao Qiang\And
Dongxiao Zhu
\affiliations
Department of Computer Science, Wayne State University, USA\\
\emails
\{xinlee, pan.deng, cli, yao, dzhu\}@wayne.edu
}

\begin{document}

\maketitle

\begin{abstract}
There are increasing demands for understanding deep neural networks' (DNNs) behavior spurred by growing security and/or transparency concerns. Due to multi-layer nonlinearity of the deep neural network architectures, explaining DNN predictions still remains as an open problem, preventing us from gaining a deeper understanding of the mechanisms. To enhance the explainability of DNNs, we estimate the input feature's attributions to the prediction task using divergence and flux. Inspired by the divergence theorem in vector analysis, we develop a novel Negative Flux Aggregation (NeFLAG) formulation and an efficient approximation algorithm to estimate attribution map. Unlike the previous techniques, ours doesn't rely on fitting a surrogate model nor need any path integration of gradients. Both qualitative and quantitative experiments demonstrate a superior performance of NeFLAG in generating more faithful attribution maps than the competing methods. Our code is available at \url{https://github.com/xinli0928/NeFLAG}.
\end{abstract}

\section{Introduction}
The growing demand for trustworthy AI in security- and safety-critic domains has motivated developing new methods to explain DNN predicitons using image, text and tabular data  \cite{chefer2021transformer,li2021improving,panijcai20,qiang2022attcat}. As noted in some pioneering works, e.g., \cite{hooker2019benchmark,smilkov2017smoothgrad,kapishnikov2021guided}, faithful explanation is indispensable for a DNN model to be trustworthy. However, it remains to be challenging for human to understand a DNN's predictions in terms of its input features due to its black-box nature.

As such, the field of explainable machine learning has emerged and seen a wide array of model explanation approaches. Among others, local approximation and gradient based methods are the two major categories that are more intensively researched. Local approximation methods mimic the local behavior of a black-box model within a certain neighborhood using some simple interpretable models, such as linear models and decision trees. However, they either require an additional model training processes (e.g., LIME\cite{ribeiro2016should}, SHAP\cite{lundberg2017unified}), or rely on some customized propagation rules (e.g., LRP\cite{bach2015pixel}, Deep Taylor Decomposition\cite{montavon2017explaining}, DeepLIFT\cite{shrikumar2017learning}, DeepSHAP\cite{chen2021explaining}). Gradient based methods such as Saliency Map \cite{simonyan2013deep}, SmoothGrad \cite{smilkov2017smoothgrad}, FullGrad \cite{srinivas2019full}, Integrated Gradient (IG) and its variants \cite{sundararajan2017axiomatic,hesse2021fast,erion2021improving,pan2021explaining,kapishnikov2019xrai,kapishnikov2021guided}  require neither surrogates nor customized rules but must tackle unstable estimates of gradients w.r.t. the given inputs. IG type of path integration based methods mitigate this issue via a path integration for gradient smoothing, however, this also introduces another degree of instability and noise sourced from arbitrary selections of baselines or integration paths. Others (e.g., Saliency Map, FullGrad) relax this requirement nevertheless can be vulnerable to the small perturbations of the inputs due to its locality. 

Ideally, we hope to avoid surrogates, special rules, arbitrary baselines and/or path integrals in interpreting DNN's prediction. In addition, since DNN interpretation often works on per sample basis, efficient algorithms are also crucial for a scalable DNN explanation technique. Here we examine the prediction behavior of DNNs through the lens of divergence and flux in vector analysis. We propose a novel Negative Flux Aggregation (NeFLAG) approach, which reformulates gradient accumulation as divergence. By converting divergence to gradient fluxes according to divergence theorem, NeFLAG interprets the DNN prediction with an attribution map obtained by efficient aggregation of the negative fluxes (see divergence theorem, flux, and divergence in Preliminaries).


We summarize our contributions as follows: 1) To the best of our knowledge, this is the first attempt to tackle the problem of explaining DNN model prediction leveraging the concepts of gradient divergence and fluxes. 2) Our NeFLAG technique eliminates the need for path integration via converting divergence to flux estimation, opening a new avenue for gradient smoothing techniques. 3) We propose an efficient approximation algorithm to enable a scalable DNN model explanation technique. 4) We investigate the relationship between flux estimation and Taylor approximation, bridging our method with other local approximation methods. 5) Both qualitative and quantitative experiment results demonstrate NeFLAG's superior performance to the competing methods in terms of interpretation quality.

\section{Related Work}
We categorize our approach as one of the local approximation methods. Other examples include LIME \cite{ribeiro2016should}, which fits a simple model (linear model or decision trees etc.) in the neighborhood of a given input point, and use this simple model as a local interpretation. Similarly, SHAP \cite{lundberg2017unified} generalizes LIME via a Shapley value reweighing scheme that has been proved to yield more consistent results. These methods, in general, enjoy their own merit when the underlying model is completely black-box, i.e., no gradient and model architecture information are available. However, in the case of DNN model interpretation, we usually know both model architecture and gradient information, so it is often beneficial to utilize this extra information as we are essentially interpreting the DNN model itself rather than a surrogate model. Saliency Map \cite{simonyan2013deep} and FullGrad \cite{srinivas2019full} exploit the gradient information, yet remain sensitive to small perturbations of the inputs. SmoothGrad \cite{smilkov2017smoothgrad} improves it by averaging attributions over multiple perturbed samples of the input by adding Gaussian noise to the original input. As we will see in Section of Preliminaries below, our NeFLAG method not only exploits the smoothing gradient information to achieve robustness against perturbations, but also eliminates the need for fitting an extra surrogate model.

Another line of research explains DNN predictions via accumulating/integrating gradients along specific paths from baselines to the given input. Here we denote them as path integration based methods, represented by Integrated Gradients (IG) \cite{sundararajan2017axiomatic}. Recent variants includes fast IG \cite{hesse2021fast}, and IG with alternative paths and baselines (e.g., Adversarial Gradient Integration (AGI) \cite{pan2021explaining}, Expected Gradient (EG) \cite{erion2021improving}), Guided Integrated Gradients (GIG)\cite{kapishnikov2021guided}, Attribution through Regions (XRAI) \cite{kapishnikov2019xrai}). IG chooses a baseline (usually a black image) as the reference to calculate the attribution by accumulating gradients along a straight-line path from the baseline to the given input in the input space. AGI, on the other hand, relaxes the baseline and straight-line assumption by utilizing adversarial attacks to automatically generate both baselines and paths, thereafter accumulating gradients along these paths. Both would require paths and baselines for gradient smoothing via path integration regardless of manually picked or generated, which could introduce attribution noises along the path (as noted by \cite{sundararajan2017axiomatic}, different paths could result in completely different attribution maps). On the contrary, our NeFLAG method doesn't need a path nor a baseline, in which the gradient smoothing is controlled by a single radius parameter $\epsilon$, opening up a new direction for gradient smoothing techniques.


\section{Preliminaries: Divergence and Flux}\label{sec:prelim}
We start explaining our approach by introducing the concept of divergence and 
flux in vector analysis. Lets first consider a general scenario: to interpret a 
DNN's prediction, we hope to characterize its local behavior. Let's define a 
DNN model $f: \boldsymbol{X}\to \boldsymbol{Y} $, which takes inputs 
$\boldsymbol{x}\in \boldsymbol{X}$, and outputs $f(\boldsymbol{x}) \in 
\boldsymbol{Y}$. For simplicity, we also assume that the model is locally 
continuously differentiable. When we query the interpretation for a given input 
$\boldsymbol{x}$, we are interested in how and why a decision is made, i.e., 
what is the underlying decision boundary. In fact, this is also the idea behind 
many other interpretation methods such as LIME\cite{ribeiro2016should} and 
SHAP\cite{lundberg2017unified}. In these methods, an interpretable linear model 
(or other simple models) is fitted via sampling additional points around the 
neighborhood of input of interest. Clearly, taking advantage of a certain kind 
of neighborhood aggregation is a promising route for interpretation of local 
approximation.

When the gradients $\nabla_{\boldsymbol{x}} f$ are available, it is already a 
decent indicator of DNN's local behavior. If we only calculate the gradients at 
$\boldsymbol{x}$, the resulting attribution map is called the Saliency 
Map\cite{simonyan2013deep}. However, without aggregation, these gradients are 
usually unstable due to adversarial effect, where a small perturbation of the 
inputs can lead to large variation of gradient values. On the other hand, these gradients may be vanishing due to the so-called saturation effect \cite{miglani2020investigating}. To overcome the 
instability, lets denote a neighborhood around input $\boldsymbol{x}$ to be 
$V_{\boldsymbol{x}}$, and estimate the average gradients over it. Intuitively, 
the resulting vector $\int_{V_{\boldsymbol{x}}}\nabla f\cdot \mathrm{d} 
V_{\boldsymbol{x}}$ can be viewed as a local approximation for the underlying 
model. 
However, neither a single gradient evaluation nor neighborhood gradient 
integration exploits the fact that
the function value $f(\boldsymbol{x})$ can be viewed as a result of gradient 
field flow accumulation. The gradient field flow therefore inspires us to 
accumulate the gradients along the flow direction, giving rise to the new idea 
of gradient accumulation.

\subsection{Gradient Accumulation}
Plenty of methods have already embraced the idea of accumulation without 
explicitly defining it. They typically accumulate the gradients along a certain 
path, for example, Integrated Gradients (IG)\cite{sundararajan2017axiomatic} 
integrates the gradients along a straightline in the input space from a 
baseline to the given input. Specifically, they study not a single input point, 
but the gradient flows from a baseline point towards the given input 
(Figure~\ref{fig:igflow}). 
The baseline here is selected assuming no information is contained for 
decision. Another example is Adversarial Gradient Integration (AGI) 
\cite{pan2021explaining}, which instead integrates along multiple paths 
generated by an adversarial attack algorithm, and aggregates all of them. This 
method also accumulates gradient flows, and it differs from IG only in 
accumulation paths and selection of baselines. 

\begin{figure}
	\centering
	\includegraphics[width=0.75\linewidth]{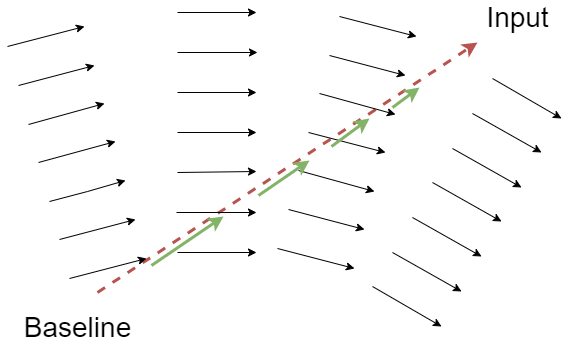}
	\caption{In a non-uniform field, IG accumulates gradient flow from a 
	baseline to the given input. Here black solid arrows denote the gradient 
	field, the red dashed arrow denotes the accumulation direction, and the 
	length of green solid arrows denote the gradient magnitude on the 
	accumulation direction.
	}
	\label{fig:igflow}
 \vspace{-4mm}

\end{figure}

Although both IG and AGI exploit the idea of gradient accumulation, they only 
tackle one-dimensional accumulation, i.e., accumulating over a path/paths. In 
fact, \cite{pan2021explaining} point out that neither an accumulation path is 
unique, nor a single path is necessarily sufficient. Ideally, an accumulation 
method should take into account of every possible baseline and path. However, 
it is nearly impossible to accumulate all possible paths. Then how could we 
overcome such an obstacle? A promising direction is to examine the problem through the lens of  
divergence and flux.

\subsection{Divergence and Flux}
For studying accumulation, the definition of divergence perfectly fits our 
needs. Let $\mathbf{F} = \nabla_x f$ be the gradient field, the divergence 
$\operatorname{div} \mathbf{F}$ is defined by
\begin{align}\label{def:div1}
\operatorname{div} \mathbf{F}=\nabla \cdot 
\mathbf{F}=\left(\frac{\partial}{\partial x_1}, \frac{\partial}{\partial x_2}, 
...\right) \cdot\left(F_{x_1}, F_{x_2}, ...\right),
\end{align}
where $F_{x_i}$ and $x_i$ are the gradient vector's and input's $i$th entry, 
respectively. The above definition is convenient for computation, however, 
difficult to interpret. A more intuitive definition of divergence can be 
described as follows.
\begin{definition}\label{def:div2}
	Given that $\mathbf{F}$ is a vector field, the divergence at position 
	$\boldsymbol{x}$ is defined by the total vector flux through an 
	infinitesimal surface enclosure $S$ that encloses $\boldsymbol{x}$, i.e.,
	\begin{align}
	\left.\operatorname{div} \mathbf{F}\right|_{\mathbf{x}}=\lim _{V 
	\rightarrow 0} \frac{1}{|V|} \oiint_{S(V)} \mathbf{F} \cdot 
	\hat{\mathbf{n}} d S,
	\end{align}
	where $V$ is an infinitesimal volume around $\boldsymbol{x}$ that is 
	enclosed by $S$, and $\mathbf{F} \cdot \hat{\mathbf{n}}$ denotes the normal 
	vector flow (flux) through the surface $S$.
\end{definition}

It is not hard to see that $\nabla \cdot \mathbf{F}$ is essentially the 
(negatively) accumulated gradients at point $\boldsymbol{x}$ from 
Definition~\ref{def:div2}, as it defines the total gradient flow contained in 
the infinitesimal volume $V$.

Therefore, given input $\boldsymbol{x}$, to interpret the model within its 
neighborhood denoted by $V_{\boldsymbol{x}}$, we only need to calculate the 
accumulated gradients by integrating the divergence over all points within this 
neighborhood. The attribution heatmap can be obtained by simply replacing the 
dot product in Eq.~\ref{def:div1} by an element-wise product, and then 
integrate, i.e.,
\begin{align}\label{eq:3}
\text{Attribution} = \int_{V_{\boldsymbol{x}}}\left(\frac{\partial}{\partial 
x_1}, \frac{\partial}{\partial x_2}, ...\right) \odot\left(F_{x_1}, F_{x_2}, 
...\right) dV.
\end{align}

Eqs. 2 and 3 are equivalent because of the divergence theorem described below. We use element-wise product because we want to disentangle the effects from different input entries. 

However, there is one major obstacle, that is, how to integrate the whole 
volume $V_{\boldsymbol{x}}$ especially when the input dimension is high. It is 
tempting to sample a few points inside the neighborhood surface enclosure, and 
sum up the divergence. But the computational cost for divergence, which 
involves second order gradient computation, is much higher than only 
calculating the first order gradients.

Fortunately, we can gracefully convert the volume integration of divergence 
into surface integration of gradient fluxes using divergence theorem, and thus 
simplify the computation.

\subsection{Divergence Theorem}
In vector analysis, the divergence theorem states that the total divergence 
within an enclosed surface enclosure is equal to the surface integral of the 
vector field's flux over such surface enclosure. The flux is defined as the 
vector field's normal component to the surface. Formally, let $\mathbf{F}$ be a 
vector field in a space $U$, $S$ is defined as a surface enclosure in such 
space, and we call the neighborhood volume that is enclosed in $S$ as $V$. 
Assuming that $\mathbf{F}$ is continuously differentiable on $V$, the 
\emph{Divergence Theorem} states that
\begin{align}
\iiint_{V}(\nabla \cdot \mathbf{F}) \mathrm{d} V=\oiint_{S}(\mathbf{F} \cdot 
\hat{\mathbf{n}}) \mathrm{d} S.
\end{align}
Here the integrating symbols $\iiint$ and $\oiint$ don't necessarily need to be 
$3$-dimensional and $2$-dimensional, respectively. They can be any higher 
dimension as long as the surface integration is one dimension lower than the 
volume integration.

In DNNs, the gradient field $\mathbf{F} = \nabla_{\boldsymbol{x}} f$ w.r.t. the 
inputs is exactly a vector field in the input space. Note that there could be 
non-differentiable points (that could violate the continuously differentiable 
condition due to some activation functions such as ReLU may not differentiable 
at some point). Nevertheless we can safely assume that it is at least 
continuously differentiable within a certain small neighborhood.

\begin{figure*}[th]
	\centering
	\includegraphics[width=0.75\linewidth]{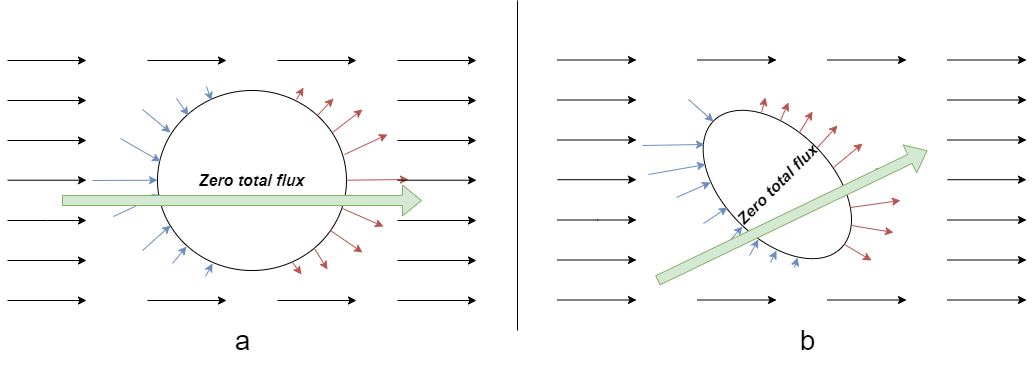}
	\caption{The fluxes on closed surfaces with different shapes. \textbf{a}. 
	The total flux of a closed sphere surface is zero when the field is 
	uniform, and the vector obtained by negative flux vector aggregation is 
	coherent to the field vector. Here red arrows and blue arrows denote 
	positive and negative flux. The green arrow bar denotes the 
	direction of the aggregated vector. \textbf{b}. When the closed surface is 
	not symmetric with respect to the field, the resulting aggregated vector is 
	not in the same direction as the field.}
	\label{fig:linearflux}
 \vspace{-4mm}

\end{figure*}

\section{Negative Flux Aggregation}\label{sec:model}
In this Section, we describe our interpretation approach, Negative Flux 
Aggregation (NeFLAG). We first explain the necessity of differentiating 
positive and negative fluxes, and then describe an algorithm on how to find 
negative fluxes on a $\epsilon$-sphere and aggregate them for interpretation.
\subsection{Negative Fluxes Define a Linear Model}
Interpreting via directly integrating all fluxes is an intriguing approach, 
however, the positive and negative fluxes should be interpreted differently. By 
convention, positive fluxes are pointed outward the enclosure surface, and the 
negative fluxes are pointed inward (Figure~\ref{fig:linearflux}a). Let $S$ be a 
surface enclosure and $V$ is its corresponding volume, a positive flux can be 
viewed as the gradient loss of $V$, and similarly a negative flux is the 
gradient gain of $V$. It means the confidence gain or loss of the prediction of 
$x$. For example, from $f(\boldsymbol{x}) = 0.99$ to $f(\boldsymbol{x}) = 0.8$ 
represents a confidence loss. The sum of negative fluxes here means a total 
confidence loss if $\boldsymbol{x}$ moves out of the neighborhood. When 
interpreting a DNN model prediction $f$, assuming $\boldsymbol{x}$ is a point 
in the input space $U$, the prediction of it by $f$ is class $A$. To interpret 
this prediction at a location in the input space (i.e., $\boldsymbol{x}$), we 
need to find out that moving to which direction could make it less likely to be 
$A$? Apparently, a good choice is to move in the direction of negative 
gradient, i.e., with negative gradient fluxes. In contrast, the positive flux 
can be interpreted as what makes it more likely to be class $A$. But since 
$\boldsymbol{x}$ is already predicted as $A$, this information is less 
informative than the negative fluxes in terms of interpretation. (An additional 
intuition is that for a given input $\boldsymbol{x}$ we are more interested in 
its direction towards decision boundary, instead of the other side of it.)


\noindent\textbf{Linear Case}: It is not hard to see that in a linear vector 
field $\mathbf{F}_l$ with underlying linear function $f_l$, the sum of fluxes 
within any enclosure is 0 (Figure~\ref{fig:linearflux}).  Moreover, when the 
enclosure is an $N$-dimensional sphere (where $N$ is the dimension of the 
space), we can sum up over all vectors on the sphere surface that are 
contributing negative flux. The obtained summed-up vector is then equivalent to 
the weight vector of the underlying linear function $f_l$. 

The observation of the linear case study inspires us to learn the weight vector 
of a local linear function by summing up all negative fluxes on an $N$-d 
sphere. Hence we say that the vector sum of all negative fluxes interprets the 
local behavior of the original function via a local linear approximation. 
Formally, we state Assumption~\ref{assum:negflux} as follows.

\begin{assumption}\label{assum:negflux}
	Given model $f$ whose gradient field is $\mathbf{F}$, and sample 
	$\boldsymbol{x}$, let $S_{\boldsymbol{x}}$ be an $N$-d $\epsilon$-sphere 
	that centered at $\boldsymbol{x}$. The vector sum of all negative flux can 
	be obtained by
	\begin{align}\label{eq:def1}
	\boldsymbol{w} = \oiint_{S_{\boldsymbol{x}}^-}\mathbf{F} \odot 
	\hat{\mathbf{n}} dS_{\boldsymbol{x}}, 
	\end{align}
	where $\odot$ denotes element-wise product, and $S_{\boldsymbol{x}}^-$ is a 
	set of point on the sphere where flux is negative. The linear model defined 
	by weight vector $\boldsymbol{w}$ is then a local linear interpretation to 
	the original model of input $\boldsymbol{x}$.
\end{assumption}

Note that this assumption is derived from the observation when $y=a_1 x_1 + a_2 
x_2+...+a_n x_n$, so we interpret this linear model by attributions $\mathbf{
	attr}=(a_1 x_1, a_2 x_2, ..., a_n x_n)$. The entry $attr_i=a_ix_i$ 
	represents the contribution of the $i$th attribute to the final prediction. 
	Note that the attribution can also be written as $\mathbf{
	attr}= \boldsymbol{a}\odot \boldsymbol{x}$, hence the element-wise product 
	in the Assumption. It is also critical to set $S_{\boldsymbol{x}}$ as a 
	sphere in Assumption~\ref{assum:negflux}, any other shape introducing any 
	degree of asymmetry would cause disagreement between local linear 
	approximation and the negative flux aggregation, as shown in 
	Figure~\ref{fig:linearflux}b.

\noindent\textbf{Interpretation:} An interpretation can be achieved by plotting 
an attribution map using the vector $\boldsymbol{w}$ from Eq.~\ref{eq:def1}. 
The rationale is directly from the correspondence between this vector and the 
underlying linear model.

\subsection{An Approximation Algorithm}
To calculate Eq.~\ref{eq:def1}, we have to find out both the point 
$\tilde{\boldsymbol{x}}$ with negative flux  and its normal vector 
$\hat{\mathbf{n}}$. For the latter, given a candidate point 
$\tilde{\boldsymbol{x}}$ on the sphere surface $S_{\boldsymbol{x}}$, we can 
replace $\hat{\mathbf{n}}$ by $(\tilde{\boldsymbol{x}} - \boldsymbol{x}) / 
|\tilde{\boldsymbol{x}} - \boldsymbol{x}|$ because $\boldsymbol{x}$ is the 
center of the sphere $S_{\boldsymbol{x}}$, and the radius 
$\epsilon=|\tilde{\boldsymbol{x}} - \boldsymbol{x}|$.
As for point $\tilde{\boldsymbol{x}}$ with negative flux, a straightforward 
solution is to randomly subsample a list of points on the $\epsilon$-sphere, 
then select those with negative flux. However, this trick has no guarantee on 
how many subsampling is sufficient, and also cannot guarantee that a negative 
flux exists. Therefore, we need an approximation algorithm that can guide us to those points 
with negative fluxes.

Since we are interested in interpreting the local behavior, meaning that the 
radius of $\epsilon$-sphere should be sufficiently small. This provides us with 
the possibility to approximate the gradient (flux). i.e.,
\begin{align}\label{eq:fluxapprox}
\mathbf{F}(\tilde{\boldsymbol{x}}) \cdot \hat{\mathbf{n}} \approx 
\frac{\left(f(\tilde{\boldsymbol{x}}) - f(\boldsymbol{x})\right)}{\epsilon}.
\end{align}
Since $\boldsymbol{x}$ is the center of $\epsilon$-sphere $S_{\boldsymbol{x}}$, 
to make the flux (left hand side of Eq.~\ref{eq:fluxapprox}) negative, we only 
need to find $\tilde{\boldsymbol{x}}$ such that $f(\tilde{\boldsymbol{x}}) < 
f(\boldsymbol{x})$. Moreover, for interpreting a classification task, it is 
usually the case that $f(\boldsymbol{x})$ is of a high value (when a prediction 
is confident, the output probability would be close to 1). Therefore, if we 
could find a random local minimum $\tilde{\boldsymbol{x}}$ on the 
$\epsilon$-sphere $S_{\boldsymbol{x}}$, it would most likely have 
$f(\tilde{\boldsymbol{x}}) < f(\boldsymbol{x})$.

In order to obtain a local minimum starting with any arbitrary initialization, 
we propose the following Lemma.

\begin{lemma}\label{lemma}
	Let $\boldsymbol{x}$ be the center of sphere $S_{\boldsymbol{x}}$, and 
	$\boldsymbol{x}^{(0)}$ be a random point on the $\epsilon$-sphere 
	$S_{\boldsymbol{x}}$. We define the following recurrence
	\begin{align}\label{eq:update}
	\boldsymbol{x}^{(k)} = \boldsymbol{x} - \epsilon \cdot 
	\frac{\mathbf{F}\left(\boldsymbol{x}^{(k-1)}\right)}{\left| 
	\mathbf{F}\left(\boldsymbol{x}^{(k-1)}\right) \right|}. 
	\end{align}
	Let $V_x$ be the space enclosed by $S_{\boldsymbol{x}}$, assuming 
	$\mathbf{F}$ is continuous on $V_x$, then $f(\boldsymbol{x}^{(k)})$ 
	converges to a local minima on $S_{\boldsymbol{x}}$ as $k\to \infty$.
\end{lemma}
\begin{proof}
	To find a local minima on the surface $S_{\boldsymbol{x}}$, given the 
	current candidate $\boldsymbol{x}^{(k-1)}$ and its corresponding gradient 
	$\mathbf{F}(\boldsymbol{x}^{(k-1)})$, the updating formula needs to take 
	$\boldsymbol{x}^{(k-1)}$ to $\boldsymbol{x}^{(k)}$ such that 
	$\boldsymbol{x}^{(k)}$ is on the opposite direction of the tangent 
	component of $\mathbf{F}(\boldsymbol{x}^{(k-1)})$ to the surface 
	$S_{\boldsymbol{x}}$. Assuming the updating step to be $\epsilon$, the 
	updating rule is then
	\begin{align}\label{eq:lemma}
	\boldsymbol{x}^{(k)} = \boldsymbol{x}^{(k-1)} - \epsilon \cdot \left( 
	\frac{\mathbf{F}\left(\boldsymbol{x}^{(k-1)}\right)}{\left| 
	\mathbf{F}\left(\boldsymbol{x}^{(k-1)}\right) \right|} - 
	\frac{\mathbf{F}\left(\boldsymbol{x}^{(k-1)}\right) \cdot 
	\hat{\mathbf{n}}}{\left| \mathbf{F}\left(\boldsymbol{x}^{(k-1)}\right) 
	\right|}  \cdot \hat{\mathbf{n}}. \right).
	\end{align}
	Note that the term in the parentheses is exactly the tangent direction of 
	the gradient (i.e., the direction of the gradient minus the normal 
	component). We can derive that
	
	\begin{align}
	\boldsymbol{x} &= \boldsymbol{x}^{(k-1)} + \epsilon\cdot 
	\frac{\left(\boldsymbol{x} - \boldsymbol{x}^{(k-1)}\right)}{\epsilon} \\ 
	\label{eq:8}
	&\approx\boldsymbol{x}^{(k-1)} + \epsilon \cdot 
	\frac{\mathbf{F}\left(\boldsymbol{x}^{(k-1)}\right) \cdot 
	\hat{\mathbf{n}}}{\left| \mathbf{F}\left(\boldsymbol{x}^{(k-1)}\right) 
	\right|}  \cdot \hat{\mathbf{n}}.,
	\end{align}
	
	where  $\boldsymbol{x}$ is the center of the $\epsilon$-sphere 
	$S_{\boldsymbol{x}}$, substituting Eq.~\ref{eq:8} into Eq~\ref{eq:lemma}, 
	we then have the updating rule in Eq.~\ref{eq:update}.
\end{proof}
Lemma~\ref{lemma} tells us that given any initial point $\boldsymbol{x}^{(0)}$ 
on the sphere, we can always find a point with local minimal flux value. 
Ideally, by seeding multiple initial points, following the recurrence in 
Eq.~\ref{eq:update}, we can obtain a set of local minima points on 
$S_{\boldsymbol{x}}$. The integration in Eq.~\ref{eq:def1} can then be 
approximated by summation of the gradient fluxes at these local minima points. 
But this approximation is suboptimal because the updating rule in 
Eq.~\ref{eq:lemma} will cause bias towards those points with \emph{minimal 
fluxes}, instead of distributing evenly to points with \emph{negative fluxes}. 
To overcome this issue, we add a $\text{sign}$ operation on the recurrence in 
Eq.~\ref{eq:update}, implemented in Algorithm \ref{alg:neflag}, to introduce 
additional stochasticity, hence making the negative flux points distributed 
more evenly. Algorithm~\ref{alg:neflag} describes this trick as well as the 
step-by-step procedure of calculating NeFLAG.

Here we note Eq.~\ref{eq:def1} provides a {\it formulation} (or a {\it framework}) to calculate a representing vector for the neighborhood near $\boldsymbol{x}$. We develop Algorithm~\ref{alg:neflag} as {\it one way of approximation} to efficiently perform the experiments. There are other algorithms with better approximation accuracy that we may explore the different approximation tricks in the future.

\begin{algorithm}[tb]
	\caption{$\text{NeFLAG}(f, \boldsymbol{x}, n, S_{\boldsymbol{x}}, \epsilon, 
	m$)}
	\label{alg:neflag}
	\textbf{Input}: $f$: Classifier,  $\quad\boldsymbol{x}$: input, 
					$\quad n$: number of negative flux samples, 			
					$S_{\boldsymbol{x}}$: $\epsilon$-sphere, $\quad\epsilon$: 
	radius of $S_{\boldsymbol{x}}$, $\quad m$: max number of backpropagation steps;\\
	\textbf{Output}: Attribution map $\text{NeFLAG}$;
	
	\begin{algorithmic}[1] 
		\STATE $\text{NeFLAG}\leftarrow 0$;
					$\quad j \leftarrow 0$; $\quad i \leftarrow 0$ ;
		\WHILE{$i=1:n$}
		\WHILE{$j=1:m$}
		\STATE Randomly sample $\tilde{\boldsymbol{x}}$ on sphere 
		$S_{\boldsymbol{x}}$;
		\STATE $\tilde{\boldsymbol{x}} = \boldsymbol{x} - \epsilon \cdot 
		\text{sign}\left(\frac{\mathbf{F}\left(\tilde{\boldsymbol{x}}\right)}{\left|
		\mathbf{F}\left(\tilde{\boldsymbol{x}}\right) \right|}\right)$ ;
		\ENDWHILE
		\STATE $\text{NeFLAG} = \text{NeFLAG} + 
		\mathbf{F}\left(\tilde{\boldsymbol{x}}\right) \odot 
		(\boldsymbol{x}-\widetilde{\boldsymbol{x}})$;
		\ENDWHILE
	\end{algorithmic}
\end{algorithm}

\begin{figure*}[t]
	\centering
	\includegraphics[width=0.85\linewidth]{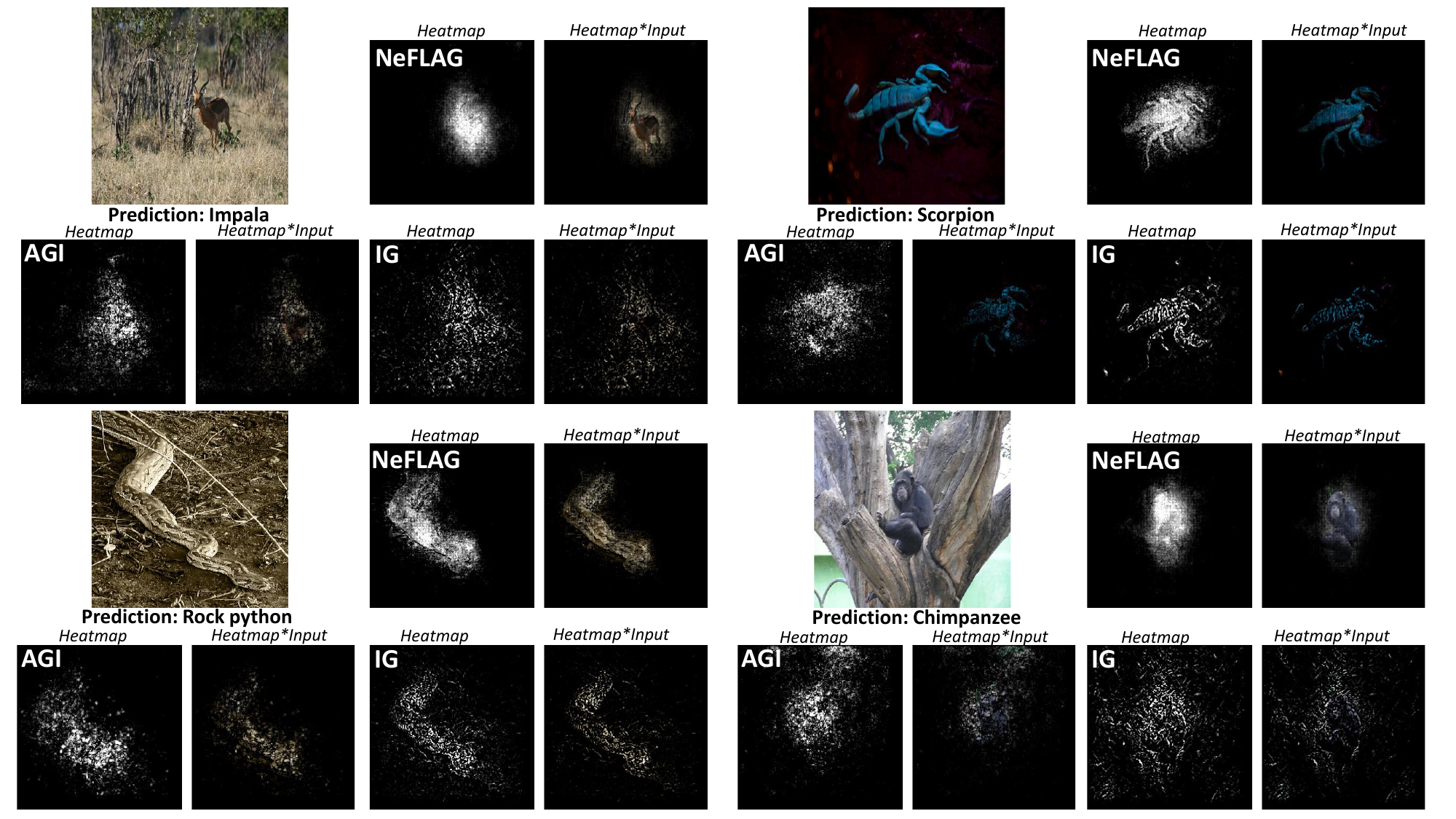}
	\caption{Examples of attribution maps obtained by NeFLAG, AGI and IG 
	methods. The underlying prediction model is InceptionV3 (additional 
	examples for ResNet152 and VGG19 can be found in the supplementary 
	materials). Compared to AGI and IG, we observe that NeFLAG's attribution map 
	has a clearer shape and focuses more densely on the target object.}
	\label{fig:ex}

\end{figure*}

\subsection{Connection to Taylor Approximation}
NeFLAG can be viewed as a generalized version of Taylor approximation. Here we 
show that NeFLAG can not only be viewed as a gradient accumulation method, but
also an extension to some local approximation methods. Given that 
$S_{\boldsymbol{x}}$ is an $N$-d $\epsilon$-sphere centered at 
$\boldsymbol{x}$, for any point $\tilde{\boldsymbol{x}}$ on the sphere surface, 
we can represent the normal vector $\hat{\mathbf{n}}$ as 
$(\tilde{\boldsymbol{x}} - \boldsymbol{x}) / \epsilon$. This inspires us to 
investigate its connection to the first order Taylor approximation. As pointed 
out by \cite{montavon2017explaining}, Taylor decomposition can be applied at a 
nearby point $\tilde{x}$, then the first order decomposition can be written by
\begin{align}\label{eq:taylor}
\begin{aligned}
f(\boldsymbol{x}) 
&=f(\widetilde{\boldsymbol{x}})+\left(\left.\nabla_{\boldsymbol{x}} f 
\right|_{\boldsymbol{x}=\widetilde{\boldsymbol{x}}}\right)^{\top} 
\cdot(\boldsymbol{x}-\widetilde{\boldsymbol{x}})+\eta, 
\end{aligned}
\end{align}
where $\eta$ is the error term of second order and higher.
We use $R(\boldsymbol{x}) = \left(\left.\nabla_{\boldsymbol{x}} f 
\right|_{\boldsymbol{x}=\widetilde{\boldsymbol{x}}}\right) 
\odot(\boldsymbol{x}-\widetilde{\boldsymbol{x}})$ to represent a heatmap 
generated by this approximation process. Note that the heatmap 
$R(\boldsymbol{x})$ completely attributes $f(\boldsymbol{x})$ if 
$f(\tilde{\boldsymbol{x}}) = 0$ and the error term $\eta$ can be omitted.

One may notice that if $\tilde{\boldsymbol{x}}$ is by chance on the N-d 
$\epsilon$-sphere, the first order Taylor approxmation of $\boldsymbol{x}$ from 
$\tilde{\boldsymbol{x}}$ is then equivalent to the flux estimation at 
$\tilde{\boldsymbol{x}}$. To explain it more clearly, we can simply view 
$-(\boldsymbol{x}-\tilde{\boldsymbol{x}}) / 
|\boldsymbol{x}-\tilde{\boldsymbol{x}}|$ as the normal vector 
$\hat{\mathbf{n}}$ in Eq.~\ref{eq:def1}, and 
$\left(\left.\nabla_{\boldsymbol{x}} f 
\right|_{\boldsymbol{x}=\widetilde{\boldsymbol{x}}}\right)$ is exactly the 
gradient vector $\mathbf{F}$ at $\tilde{\boldsymbol{x}}$. Hence the first order 
Taylor decomposition and a single point flux estimation differs only by a 
factor of $-\epsilon$ as $\epsilon=|\boldsymbol{x}-\tilde{\boldsymbol{x}}|$.

Since we only care about the negative fluxes, then the second term of 
Eq~\ref{eq:taylor}, i.e.,  $\left(\left.\nabla_{\boldsymbol{x}} f 
\right|_{\boldsymbol{x}=\widetilde{\boldsymbol{x}}}\right)^{\top} 
\cdot(\boldsymbol{x}-\widetilde{\boldsymbol{x}})$ must be positive. If we 
further omit the error term $\eta$, we must have $f(\boldsymbol{x}) = 
f(\tilde{\boldsymbol{x}}) + \text{positive value} $. It tells us 
that in the perspective of first order Taylor decomposition, the negative flux 
indeed attributes to positive predictions, bridging the connection between 
NeFLAG and other local approximation methods.

\begin{table*}[]
\centering
\begin{tabular}{@{}c|c|ccccc|ccc@{}}
\toprule
Metrics                                                                    & Methods/Models     & IG    & SG    & GIG   & EG    & AGI   & NeFLAG-1       & NeFLAG-10 & NeFLAG-20      \\ \midrule
\multirow{3}{*}{\begin{tabular}[c]{@{}c@{}}Deletion\\ Score $\downarrow$ \end{tabular}}  & VGG19       & 0.071 & 0.065 & 0.054 & 0.041 & 0.040 & \textbf{0.034} & 0.052     & 0.059          \\
                                                                           & ResNet152   & 0.132 & 0.091 & 0.090 & 0.066 & 0.068 & \textbf{0.055} & 0.076     & 0.081          \\
                                                                           & InceptionV3 & 0.122 & 0.079 & 0.096 & 0.049 & 0.059 & \textbf{0.048} & 0.066     & 0.068          \\ \midrule
\multirow{3}{*}{\begin{tabular}[c]{@{}c@{}}Insertion\\ Score $\uparrow$ \end{tabular}} & VGG19       & 0.223 & 0.312 & 0.304 & 0.338 & 0.401 & 0.416          & 0.521     & \textbf{0.535} \\
                                                                           & ResNet152   & 0.332 & 0.380 & 0.437 & 0.447 & 0.480 & 0.485          & 0.568     & \textbf{0.578} \\
                                                                           & InceptionV3 & 0.375 & 0.465 & 0.465 & 0.478 & 0.480 & 0.544          & 0.618     & \textbf{0.625} \\ \bottomrule
\end{tabular}

\caption{Quantitative evaluation using deletion and insertion scores. NeFLAG-1, 
	NeFLAG-10 and NeFLAG-20 are NeFLAG methods with various number of negative flux samples. Note {\bf NeFLAG-1} already {\bf outperforms all} the competing methods with a {\bf linear computational complexity} of the number of backprogation steps (see Appendix on Computational Complexity and Overhead for detailed analysis).}
\label{table:quantitative}
\vspace{-5mm}
\end{table*}
	

\section{Experiments}
In this Section, we demonstrate and evaluate NeFLAG's performance both 
qualitatively and quantitatively using a diverse set of DNN architectures and 
baseline interpretation methods. 

\subsection{Experiment Setup}
\noindent\textbf{Models:} InceptionV3 \cite{szegedy2015rethinking}, ResNet152 
\cite{he2015deep} and VGG19 \cite{simonyan2014very} are selected as the 
pre-trained DNN models to be explained.

\noindent\textbf{Dataset:} The ImageNet \cite{imagenet_cvpr09} dataset is used 
for all of our experiments. ImageNet is a large and complex data set (compared with smaller and simpler data sets such as Places, CUB200, or Flowers102) for us to better demonstrate the key advantages of our approach. 

\noindent\textbf{Baseline Interpretation Methods:} IG 
\cite{sundararajan2017axiomatic} and AGI \cite{pan2021explaining} are selected 
as baselines for both qualitative and quantitative interpretation performance 
comparison. We also include Expected Gradients (EG)\cite{erion2021improving}, Guided Integrated Gradient \cite{kapishnikov2021guided}, and SmoothGrad \cite{smilkov2017smoothgrad} in 
quantitative comparison as they can be viewed as smoothed versions of attribution methods. We use the default setting provided by captum \footnote{\url{https://captum.ai/}} for IG, EG and SmoothGrad methods. For AGI, we adopt the default parameter settings reported in \cite{pan2021explaining}, i.e., step size $\epsilon=0.05$, number of false classes $n=20$. For GIG, we also used the default settings provided in their latest official implementation \cite{kapishnikov2021guided}. Here we focus on the comparison with {\it gradient based methods} since non-gradient methods typically require a surrogate and additional optimization processes. 


\begin{figure*}[t]
	\centering
	\includegraphics[width=0.90\linewidth]{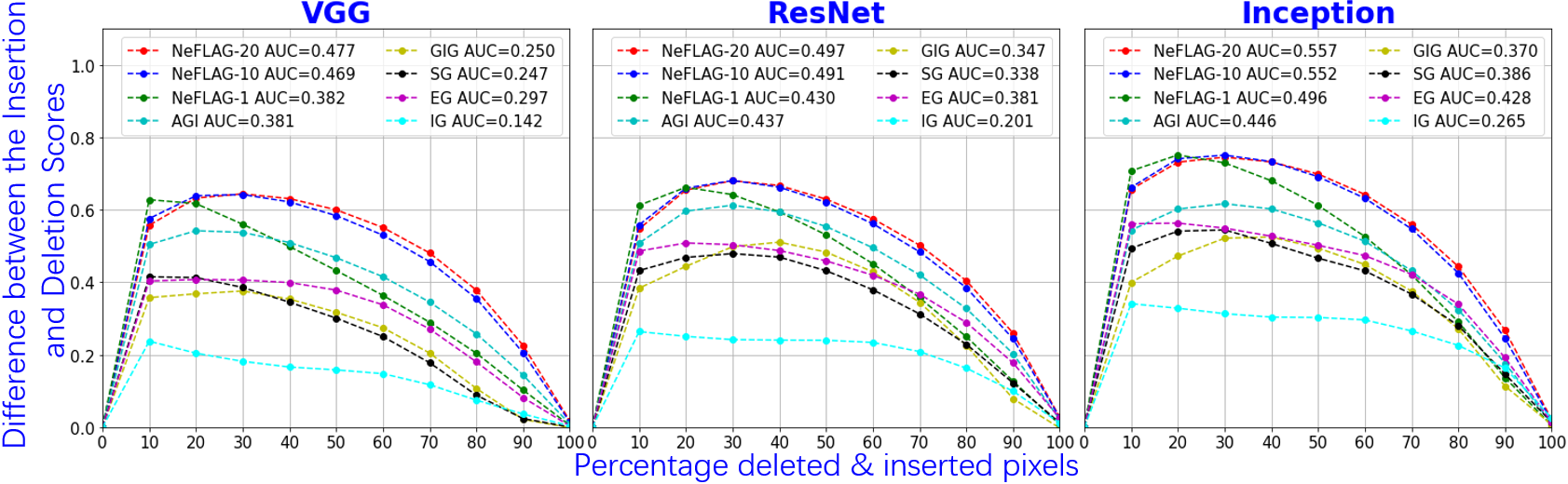}
	\caption{Quantitative Performance Comparison in terms of Difference between the Insertion and Deletion Scores.}
	\label{fig:comparison}
 \vspace{-5mm}
\end{figure*}

\subsection{Qualitative Evaluation}
We first show some examples to demonstrate a better quality of the
attribution maps generated by NeFLAG than by other baselines. The NeFLAG is configured as follows: $\epsilon$-sphere radius is set 
to $\epsilon=0.1$, and the number of random negative flux point 
$\tilde{\boldsymbol{x}}$ is $n=20$. Figure~\ref{fig:ex} shows the attribution 
maps generated by various interpretation methods for the Inception V3 model (examples 
for ResNet152 and VGG19 models can be found in the Appendix). 
Qualitatively speaking, we denote an attribution map is faithful to model's prediction if 1) it focuses on 
the target objects, and 2) has clear and defined shapes w.r.t. class label instead of sparsely distributed 
pixels. Based on this notion, it is clearly observed that NeFLAG has a better 
attribution heatmap quality than AGI and IG. 

A key observation from Figure~\ref{fig:ex} is that the NeFLAG attribution heatmap provides a defined shape of the entity of the class whereas IG and AGI attribution heatmaps do not reflect that ability. In fact, the latter heatmaps are scattered and do not retain a definite shape. For example, in Figure~\ref{fig:ex}, IG and AGI’s heatmaps fail to delineate a definite and slender shape of the \textit{Rock Python}. The same notion happens with the \textit{Impala} and \textit{Chimpanzee} where the attribution heatmaps of IG and AGI are scattered across the background and class entity. We note that the results shown in Figure~\ref{fig:ex} are very typical (not cheery-picked) in our experiments. We also provide more examples from each class to demonstrate the superior quality of attribution maps in the Appendix.

We note that the striking performance disparities referenced above could mainly be caused by the inconsistency of path accumulation methods. Since both IG and AGI incorporate a certain kind of path integration, any gradient 
vector on that path is taken into account of the final interpretation. However, 
the gradient vectors on the path aren't necessarily unique and consistent. The 
rationale behind this argument is that neither IG nor AGI is capable of 
guaranteeing that their choice of accumulation path is optimal. For example, in 
IG, the path is a straight-line, which isn't by any means ideal. Similarly in 
AGI, the path of an adversarial attack is chosen. However, due to the stochastic 
effect, a slightly different initialization could result in completely 
different attack paths. Our method NeFLAG, on the other hand, takes advantage 
of negative flux aggregation without the burden of path accumulation. 

In fact, NeFLAG's advantage is highly pronounced when we set the number of negative flux points to 1 (i.e., $n=1$), where we will show hereinafter using quantitative experiments, is sufficient for a superior performance to the competing methods. Yet, NeFLAG-1 achieves an unequaled computational efficiency via a single point gradient evaluation instead of a path integration used in IG type methods. 


\subsection{Quantitative Evaluation}\label{sec:quantitative}

In the quantitative experiments, we compare the performance of IG, AGI, EG, GIG, SmoothGrad and three NeFLAGs variants with different numbers of sampled negative flux points. Since our NeFLAG is considered as a {\it robust local attribution} method, the size the neighborhood (radius of $\epsilon$-sphere) and number of negative fluxes are tuning parameters. Similarly, the baseline local attribution method, SmoothGrad, also has two tuning parameters: the noise level and the number of samples to average over. Contrastively, IG, AGI, EG, and GIG are {\it global attribution} methods so there is no tuning parameter involved. We extensively investigate the neighborhood size and experiment with various number of negative fluxes (e.g., NeGLAG-1, NeGLAG-10, and NeGLAG-20) to faithfully demonstrate the stability of our method in comparison with others.

Here VGG19, InceptionV3, and ResNet152 are used as underlying DNN prediction models. Since we focus on explaining DNN prediction on a {\it per sample basis} just like other DNN explanation methods, we randomly select 5,000 samples from ImageNet validation dataset with 5 samples from each class as a good representation of the classes. We use insertion score and deletion score as our evaluation metrics\cite{petsiuk2018rise}. We replace the top pixels with black pixels in the first round \cite{petsiuk2018rise} whereas with Gaussian blurred pixels in the second round \cite{sturmfels2020visualizing}. We report the average performance of two rounds of experiments. Experimental details can be found in the Appendix.




Table~\ref{table:quantitative} demonstrates that NeFLAG outperforms other baselines even with only a single negative flux point (NeFLAG-1), and becomes much better when we incorporate sufficient amount of negative flux points (NeFLAG-10 and NeFLAG-20). In Figure \ref{fig:comparison}, we systematically compared the performance of our NeFLAG with the competing methods on three DNN architectures. We used the Difference between the Insertion and Deletion scores as the metric for comparison \cite{shah2021input}. It is because the deletion and insertion scores can be influenced by distribution shift caused by removing/adding pixels \cite{hooker2019benchmark}. Since the shift occurs during both pixel insertion and deletion, focusing on their relative difference instead of their absolute values helps in neutralizing the effects of distribution shift. It is clearly observed from Figure \ref{fig:comparison} that our NeFLAG method outperforms all the competing methods across the three DNN models in terms of the Difference between the Insertion and Deletion scores. We note there still a room for more reliable and comprehensive evaluation metrics for the attribution methods.      


\section{Conclusion}
We develop a novel DNN model explanation technique NeFLAG built off the concept of divergence, and point out a connection to other local approximation methods. NeFLAG doesn't require a baseline, nor an integration path. This is achieved by converting a volume integration of the second order gradients to a surface integration of the first order gradients using divergence theorem. Both qualitative and quantitative experiments demonstrate a superior performance of NeFLAG in explaining DNN predictions over the strong baselines.

\newpage

\bibliographystyle{named}
\bibliography{ijcai23}

\clearpage
\appendix

\section{Appendix }
\subsection{Additional experiment details}
\paragraph{Visualization:} In order to better visualize the attribution map obtained by the interpretation methods, we adopt the attribution processing method in \cite{pan2021explaining} to optimize the attribution map visualization quality, i.e., for all  attribution values that are smaller than $80\%$ percentile, we assign them as the lower bound $q$. Similarly, for all values that are larger than $99\%$ percentile, we assign them to be the upper bound $u$. Finally, all values are normalized within $[0, 1]$ for visualization purpose.

\subsection{Insertion score and deletion score:} According to \cite{petsiuk2018rise}, insertion score and deletion score can be used to evaluate explanation qualities of different attribution methods. The idea of insertion score sources from the insertion game: starting from a black image (the original version), we sequentially insert pixels of the input image based on the values of the attribution map, the pixels with larger attribution values are inserted first. At each step of this insertion game, a prediction score can be calculated using the prediction model. Hence from the first pixel to the last, there is a function curve $g$ to describe the insertion game. The insertion score is then obtained by calculating the area under the curve (AUC) of $g$. Similarly, the deletion score can also be defined and obtained.

In addition to using a black image, we ablate by replacing each pixel with its Gaussian-blurred counterpart ($\sigma = 20$) for the input as in \cite{sturmfels2020visualizing}. The intuition is that eliminating the image's most salient pixels can result in high-frequency edge artifacts, as well as the associated distribution shift problem \cite{hooker2019benchmark}, which can be mitigated by this Gaussian-blurring technique. As the result, we report the average performance of two rounds of deletion/insertion games, using black and blurred images respectively.

\subsection{Ablation study on our NeFLAG}

We analyze the effect of varying $\epsilon$-sphere radius values and the number of negative fluxes ($n$). There is another parameter $m$ which represents the max number of backpropagation steps. However, we do not include it in this Ablation Study since it has little impact on the result as the average actual number of steps required is much less than the set $m = 20$. 

From Table \ref{table:ablation}, it can be clearly observed that NeFLAG-1, regardless of the neighborhood size $\epsilon$, already outperforms all the competing methods in terms of the Difference btw. Insertion and Deletion Score (shown in Figure \ref{fig:comparison} of the main text). This is translated into the fact that our NeFLAG algorithm achieves a robust yet superior performance with an efficient computational cost (only 1 flux sample is used). When increasing the number of flux samples from 1 (NeFLAG-1) to 10 (NeFLAG-10) and 20 (NeFLAG-20), the performance of our NeFLAG algorithm monotonically increases. This ablation study provides a strong evidence that the superior performance of our NeFLAG generalizes well to various hyper-parameter settings and across different DNN architectures.    

\begin{table*}[ht]
\centering
\begin{tabular}{cc|ccc|ccc|ccc}
\hline
Metrics                                                                     & Methods     & \multicolumn{3}{c|}{NeFLAG-1}  & \multicolumn{3}{c|}{NeFLAG-10} & \multicolumn{3}{c}{NeFLAG-20}                    \\ \hline
\multicolumn{2}{c|}{$\epsilon$-sphere radius}                                             & 0.05  & 0.1            & 0.2   & 0.05     & 0.1      & 0.2      & 0.05           & 0.1            & 0.2            \\ \hline
\multirow{3}{*}{\begin{tabular}[c]{@{}c@{}}Deletion\\ Score $\downarrow$ \end{tabular}}   & VGG19       & 0.046 & \textbf{0.034} & 0.045 & 0.054    & 0.052    & 0.054    & 0.059          & 0.059          & 0.062          \\
                                                                            & ResNet152   & 0.080 & \textbf{0.055} & 0.078 & 0.084    & 0.076    & 0.080    & 0.087          & 0.081          & 0.084          \\
                                                                            & InceptionV3 & 0.071 & \textbf{0.048} & 0.065 & 0.074    & 0.066    & 0.064    & 0.077          & 0.068          & 0.066          \\ \hline
\multirow{3}{*}{\begin{tabular}[c]{@{}c@{}}Insertion\\ Score $\uparrow$ \end{tabular}}  & VGG19       & 0.438 & 0.416          & 0.430 & 0.505    & 0.521    & 0.515    & 0.538          & \textbf{0.535} & 0.531          \\
                                                                            & ResNet152   & 0.509 & 0.485          & 0.513 & 0.549    & 0.568    & 0.567    & 0.556          & \textbf{0.578} & 0.577          \\
                                                                            & InceptionV3 & 0.555 & 0.544          & 0.576 & 0.607    & 0.618    & 0.614    & 0.614          & \textbf{0.625} & 0.624          \\ \hline
\multirow{3}{*}{\begin{tabular}[c]{@{}c@{}}Difference\\ Score $\uparrow$\end{tabular}} & VGG19       & 0.393 & 0.382          & 0.385 & 0.451    & 0.469    & 0.461    & \textbf{0.479} & 0.476          & 0.469          \\
                                                                            & ResNet152   & 0.429 & 0.430          & 0.435 & 0.465    & 0.492    & 0.487    & 0.469          & \textbf{0.497} & 0.493          \\
                                                                            & InceptionV3 & 0.484 & 0.496          & 0.511 & 0.534    & 0.552    & 0.550    & 0.537          & 0.557          & \textbf{0.558} \\ \hline
\end{tabular}
\caption{Ablation study using Deletion and Insertion scores and their Difference.}
\label{table:ablation}
\end{table*}

\subsection{Computational Complexity and Overhead}

{\bf Computational Complexity:} With a pre-trained DNN model, for each image whose class to be predicted by the DNN, there are two variables controlling the computational complexity, $m$ represents the number of backpropagation steps to evaluate its gradients, and $n$ represents the number of reference samples. 

$\mathcal{O}(mn)$: The global path integration methods EG and GIG require multiple integration paths to calculate the attribution maps. So their computational complexity is $\mathcal{O}(mn)$, representing a group of methods with a quadratic complexity.    

$\mathcal{O}(m)$: The global path integration methods IG and AGI, only one reference sample is required ($n = 1$). Similarly, only one flux sample is required for our local approximation method NeFLAG-1. Note from Table 1 of the main text, NeFLAG-1 already outperforms all the competing methods. That is, our NeFLAG achieves the best performance among all with a comparable computational overhead to the efficient IG method. Therefore, the computational complexity for IG, AGI and NeFLAG-1 is $\mathcal{O}(m)$, representing a group of methods with a linear complexity.    

$\mathcal{O}(n)$: The baseline local approximation method SmoothGrad (SG), it requires multiple reference samples from the neighborhood to average over the attribution maps with only 1 backprogation to evaluate the gradients. So its computational complexity is $\mathcal{O}(n)$, representing another group of methods with a linear complexity.   

\noindent{\bf Computational Overhead:} To evaluate the empirical computational overhead, we list the hyper-parameter values of each method that yields the claimed performance as follows. IG: $n = 1, m = 100$, SG: $n = 20, m = 1$, GIG: $n = 1, m = 20$, EG: $n = 5, m = 50$, AGI: $n = 1, m = 20$, NeFLAG-1: $n = 1, m \ll 20$, NeFLAG-10: $n = 10, m \ll 20$, NeFLAG-20: $n = 20, m \ll 20$ where the $m*n$ is the total backpropagation steps required for each method. For our NeFLAG algorithm used in real experiments, it usually takes much less than 20 steps, especially when the $\epsilon$ is large. For example, when $\epsilon = 0.1$, the average $m$ values for each image in the ImageNet validation set are only 1.716 (VGG19), 1.661 (ResNet152), and 1.934 (InceptionV3), much lower than the maximum setting of 20. As a result, \textbf{NeFLAG-1 has the lowest computational overhead among all the competing methods.}

\subsection{Additional qualitative evaluation examples}
Additional examples using InceptionV3, ResNet-152 and VGG-19 can be found in the remaining pages.

\clearpage
\includepdf[pages=-]{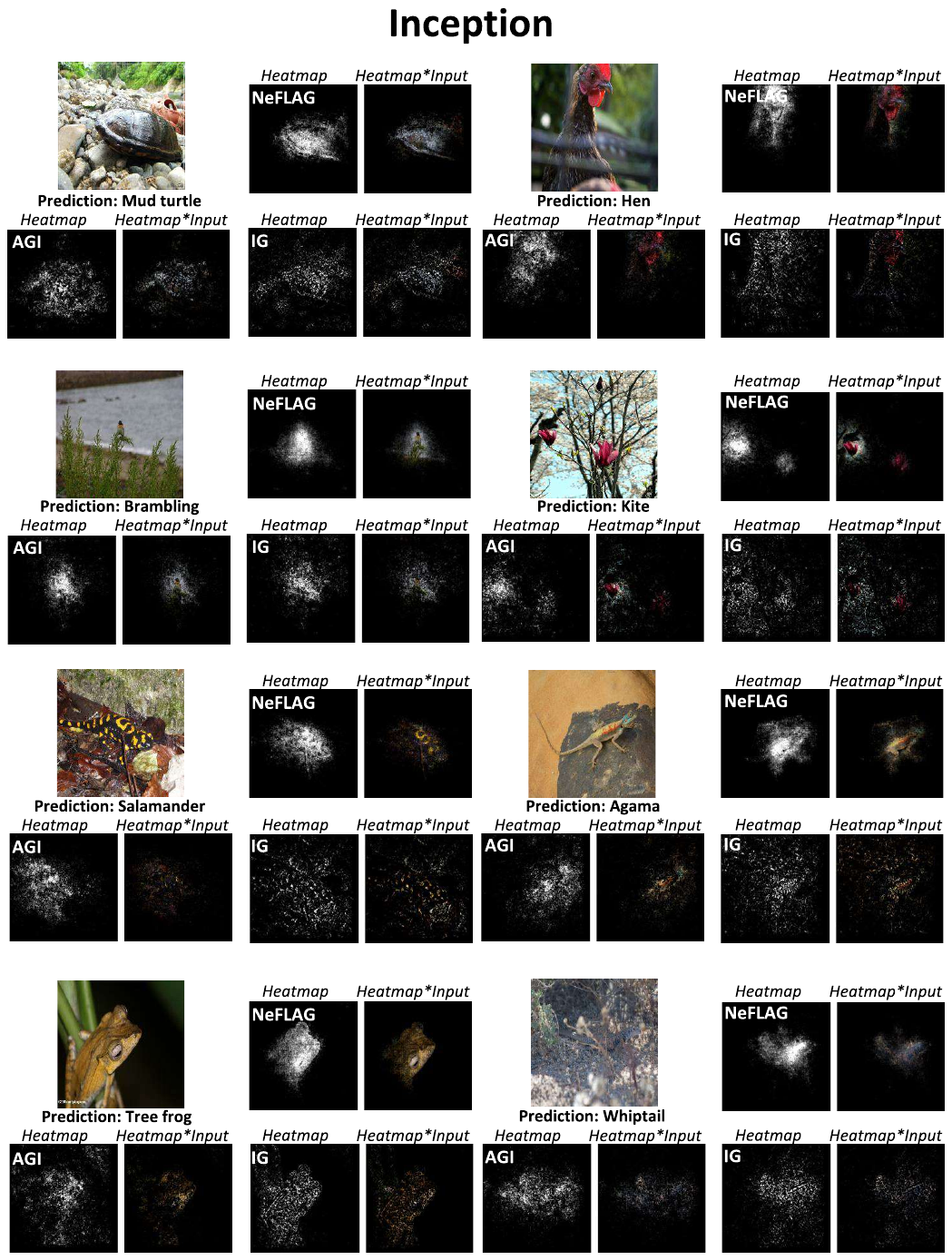}
\end{document}